\documentclass[journal]{IEEEtran}

\usepackage[T1]{fontenc}
\usepackage{lmodern}

\usepackage{amsmath,amssymb,amsthm}
\usepackage{algorithm}
\usepackage{algorithmic}
\usepackage{graphicx}
\usepackage{hyperref}
\usepackage{cleveref}

\theoremstyle{plain}
\newtheorem{theorem}{Theorem}[section]
\newtheorem{lemma}[theorem]{Lemma}
\newtheorem{corollary}[theorem]{Corollary}

\theoremstyle{definition}
\newtheorem{definition}[theorem]{Definition}
\newtheorem{axiom}{Axiom}

\theoremstyle{remark}
\newtheorem{remark}[theorem]{Remark}

\newcommand{\R}{\mathbb{R}}
\newcommand{\E}{\mathbb{E}}
\newcommand{\KL}{\text{KL}}
\newcommand{\softmax}{\text{softmax}}
\newcommand{\argmax}{\text{argmax}}
\DeclareMathOperator{\Var}{Var}

\begin{document}

\title{Sparse Knowledge Distillation: A Mathematical Framework for Probability-Domain Temperature Scaling and Multi-Stage Compression}

\author{Aaron R. Flouro and Shawn P. Chadwick, PhD.\\
research@sparse-tech.com}

\maketitle

\begin{abstract}
We develop a unified theoretical framework for sparse knowledge distillation based on probability-domain softening operators. While the equivalence $p^{1/T} \propto \softmax(z/T)$ is well known, our contribution is an operator-level analytical framework built on this foundation rather than the equivalence itself.

The framework comprises four core components: (i) operator-agnostic bias--variance decompositions that characterize when sparse students outperform dense teachers, (ii) a homotopy path formalization of multi-stage pruning in function space explaining why iterative compression succeeds where one-shot pruning fails, (iii) convergence guarantees establishing $\mathcal{O}(1/n)$ rates for $n$-stage distillation with explicit parameter dependence, and (iv) equivalence class characterizations identifying distinct probability-domain operators that yield identical student models under capacity constraints.

We introduce an axiomatic definition of probability-domain softening operators based on ranking preservation, continuity, entropy monotonicity, identity, and boundary behavior, and show that multiple non-equivalent operator families satisfy these axioms. All learning-theoretic guarantees are shown to hold uniformly across this operator class, independent of implementation details. These results provide theoretical grounding for black-box teacher distillation, partial-access settings such as top-$k$ truncation and text-only outputs, and privacy-preserving model compression.
\end{abstract}


\subsection*{Why This Framework Is Needed}
While the probability-domain equivalence $p^{1/T} \propto \softmax(z/T)$ is mathematically straightforward, knowledge distillation lacks a unified theoretical framework addressing fundamental questions: When do sparse students outperform dense teachers? Why does iterative pruning succeed where one-shot pruning fails? What convergence guarantees exist for multi-stage compression? This paper provides such a framework, with the probability-domain equivalence serving as an enabling foundation rather than the central contribution.

\section{Introduction}

\subsection{Motivation and Scope}

Knowledge distillation (KD) has emerged as a powerful technique for transferring knowledge from large teacher models to smaller student models~\cite{hinton2015distilling}. While the classical formulation relies on access to teacher logits, the equivalence $p^{1/T} \propto \softmax(z/T)$ enables probability-domain distillation. This equivalence is well-known and follows directly from softmax properties.

Our contribution is not the probability-domain equivalence itself, but rather a unified analytical framework built upon it. Recent work has argued that hallucination should be understood as a failure of internal world modeling and related distributional inconsistencies, motivating probability-level interventions rather than purely surface-level heuristics~\cite{liu2025unified_hallucination}. Complementary perspectives emphasize structural drivers (e.g., incentive vs.\ ontology framing) and clarify why simple post-hoc fixes are insufficient to fully characterize the phenomenon~\cite{ackermann2025incentives_ontology}. These analyses motivate our focus on operator-level guarantees for probability-domain distillation under partial-access constraints.

The framework addresses settings where full probability access may be unavailable:
\begin{itemize}
\item Text-only teacher outputs (no numerical probabilities)
\item Top-$k$ truncated distributions from API-based models
\item Privacy-preserving settings requiring distribution perturbation
\end{itemize}

The central questions we address are theoretical: What conditions govern when sparse students outperform dense teachers? How should multi-stage compression be analyzed?

\subsection{Primary Contributions}

This paper develops a unified framework for sparse knowledge distillation comprising four components:

\begin{enumerate}
\item \textbf{Bias--variance analysis}: Operator-agnostic decompositions characterizing when variance reduction from pruning exceeds distillation-induced bias (Section~\ref{sec:bias-variance})
\item \textbf{Homotopy formalization}: Multi-stage pruning as continuous paths in function space, explaining iterative pruning success (Section~VIII)
\item \textbf{Convergence guarantees}: $\mathcal{O}(1/n)$ rates for $n$-stage distillation with explicit parameter dependence (Section~IX)
\item \textbf{Equivalence classes}: Characterization of operator families producing identical student models (Section~X)
\end{enumerate}

The axiomatic foundation (Sections~IV--V) establishes notation and the class of operators over which these results hold. The probability-domain equivalence is the enabling foundation; the analytical framework is the contribution.

\textbf{Paper Structure.} Section II establishes notation and the probability-domain foundation. Section III provides a technical overview. Sections IV--V introduce the axiomatic framework and prove existence and non-uniqueness. Sections VI--IX establish operator-agnostic learning guarantees, including shift-invariance, bias--variance tradeoffs, homotopy-based compression, and convergence. Section X characterizes equivalence classes of operators, and Section XI discusses implications.

\section{Probability-Domain Foundation and Notation}

This section establishes notation and the foundation for subsequent analysis. The probability-domain equivalence presented here is not novel; it follows directly from softmax properties and has been observed in various forms in the literature.

\subsection{The Standard Equivalence}

Classical temperature scaling operates on logits $z \in \R^V$:
\begin{equation}
p_i^{(T)} = \frac{\exp(z_i/T)}{\sum_j \exp(z_j/T)}
\end{equation}

The equivalence $p^{1/T} \propto \softmax(z/T)$ follows from the identity $\exp(z_i/T) = \exp(z_i)^{1/T} = p_i^{1/T} \cdot (\sum_j \exp(z_j))^{1/T}$, where the normalizing constant cancels. This enables probability-domain temperature scaling without logit access.

\textbf{Notation.} We denote probability-domain softening operators as $F_T: \Delta^V \to \Delta^V$, where $\Delta^V$ is the probability simplex. The power-scaling operator $F_T(p)_i = p_i^{1/T} / \sum_j p_j^{1/T}$ is one instantiation; our framework characterizes the broader class.

\subsection{Where the Equivalence Fails: Partial Access Settings}

The standard equivalence assumes full probability access. Genuine difficulty arises in partial-access settings:

\begin{itemize}
\item \textbf{Text-only outputs}: Teacher provides sampled tokens, not probabilities
\item \textbf{Top-$k$ truncation}: APIs return only the $k$ highest-probability tokens
\item \textbf{Perturbed distributions}: Privacy mechanisms add noise to outputs
\end{itemize}

\subsection{Top-$k$ Truncated Distributions}

When APIs return only top-$k$ probabilities $\{p_{(1)}, \ldots, p_{(k)}\}$ with $\sum_{i=1}^k p_{(i)} < 1$, the power-scaling operator $p^{1/T}$ cannot be directly applied to the missing mass. Alternative approaches include:

\begin{itemize}
\item \textbf{Tail modeling}: Assume a parametric form (e.g., Zipf) for the truncated tail
\item \textbf{Renormalization}: Apply $p^{1/T}$ to visible tokens and redistribute mass
\item \textbf{Entropy matching}: Find the maximum-entropy completion consistent with observed probabilities
\end{itemize}

These partial-access settings motivate the axiomatic framework: we seek to characterize which operator properties are essential for distillation, enabling principled handling of incomplete information. The subsequent analysis (Sections~III--X) provides this characterization.

\section{Technical Overview}

This paper develops a three-layer framework for understanding probability-domain knowledge distillation:

\subsection{Layer 1: Axiomatic Foundation}

We define five axioms (Section~\ref{sec:axioms}) that any probability-domain softening operator must satisfy:
\begin{itemize}
\item Axiom 1 (Ranking): $p_i > p_j \Rightarrow F_T(p)_i > F_T(p)_j$
\item Axiom 2 (Continuity): $F_T(p)$ continuous in $p$ and $T$
\item Axiom 3 (Entropy): Higher $T$ yields higher entropy
\item Axiom 4 (Identity): $F_1(p) = p$ (reference point)
\item Axiom 5 (Boundaries): Limits at $T \to 0^+$ (one-hot) and $T \to \infty$ (uniform)
\end{itemize}

These axioms formalize intuitive properties: softening should smooth distributions (entropy), preserve relative orderings (ranking), behave predictably (continuity), and have natural limiting cases (boundaries).

\subsection{Layer 2: Operator Theory}

Section~\ref{sec:existence} proves existence of operators satisfying all axioms via three construction principles:
\begin{enumerate}
\item Entropy projection: Minimize $\KL(q \| p)$ subject to $H(q) = h(T)$
\item Power transforms: Apply $q_i \propto \phi(p_i; T)$ for monotone $\phi$
\item Convex mixing: Interpolate $q = \alpha(T) \cdot p + (1-\alpha(T)) \cdot u$ toward uniform $u$
\end{enumerate}

These constructions prove non-uniqueness: multiple distinct operators satisfy the axioms. Section~\ref{sec:equivalence} characterizes when operators are KD-equivalent (produce identical student models).

\subsection{Layer 3: Learning Theory}

Sections~\ref{sec:bias-variance}--\ref{sec:convergence} establish that any operator $F_T$ satisfying the axioms enables:
\begin{itemize}
\item Bias-variance decomposition (Theorem~\ref{thm:bias-variance}): $\E[\ell(S)] = \text{Bias}^2(S; F_T) + \Var(S) + \sigma^2$
\item Sparse optimality (Theorem~\ref{thm:sparse-optimality}): Pruned student outperforms dense teacher when $\Delta\Var > \Delta\text{Bias}^2$
\item Homotopy convergence (Theorem~\ref{thm:convergence}): Multi-stage distillation achieves $\E[\ell(S_n)] \leq \E[\ell(T)] + \mathcal{O}(1/n)$
\end{itemize}

The key insight: these guarantees hold for all conforming operators, not just one specific formula. The axioms are sufficient to ensure convergence.

\subsection{Practical Implications}

This framework provides theoretical support for:
\begin{itemize}
\item API-based distillation: Extracting knowledge from commercial large language models using only probability outputs
\item Privacy preservation: Avoiding transmission of sensitive logit information
\item Flexible implementation: Selecting operators based on computational constraints, with convergence guarantees applying uniformly
\end{itemize}

\section{Axioms for Probability-Domain Softening}
\label{sec:axioms}

Let $\Delta^V$ denote the $(V-1)$-dimensional probability simplex:
\begin{equation}
\Delta^V = \{p \in \R^V : p_i \geq 0, \sum_i p_i = 1\}
\end{equation}

\begin{definition}[Probability-Domain Softening Operator Family]
\label{def:softening-family}
A family of operators $\{F_T\}_{T>0}$ where $F_T: \Delta^V \to \Delta^V$ is a probability-domain softening family if it satisfies the following axioms:
\end{definition}

\begin{axiom}[Ranking Preservation]
\label{axiom:ranking}
For all $p \in \Delta^V$ and $T > 0$:
\begin{equation}
p_i > p_j  \Longrightarrow  F_T(p)_i > F_T(p)_j
\end{equation}
\end{axiom}

The operator preserves the relative ordering of probabilities. The most probable token remains most probable after softening.

\begin{axiom}[Continuity]
\label{axiom:continuity}
The operator $F_T(p)$ is jointly continuous in both $p$ and $T$.
\end{axiom}

Small changes in input probabilities or temperature produce small changes in output. No discontinuous jumps.

\begin{axiom}[Monotonic Entropy]
\label{axiom:entropy}
For $T' > T > 0$:
\begin{equation}
H(F_{T'}(p)) \geq H(F_T(p))
\end{equation}
where $H$ denotes Shannon entropy: $H(p) = -\sum_i p_i \log p_i$.
\end{axiom}

Higher temperature produces higher entropy (more uniform) distributions. This captures the softening property.

\begin{axiom}[Identity at Unity]
\label{axiom:identity}
For all $p \in \Delta^V$:
\begin{equation}
F_1(p) = p
\end{equation}
\end{axiom}

At temperature $T=1$ (base temperature), the operator is the identity. This establishes $T=1$ as the canonical reference point.

\begin{axiom}[Boundary Behavior]
\label{axiom:boundary}
For all $p \in \Delta^V$:
\begin{align}
\lim_{T\to\infty} F_T(p) &= \text{uniform distribution} \\
\lim_{T\to 0^+} F_T(p) &= \text{one-hot at } \argmax(p)
\end{align}
When $\argmax(p)$ is not unique (i.e., multiple indices share the maximum probability), the $T \to 0^+$ limit concentrates mass uniformly among the tied maximal indices.
\end{axiom}

As temperature approaches infinity, all tokens become equally probable (maximum entropy). As temperature approaches zero, only the most probable token(s) survive (minimum entropy). In case of ties, mass is distributed equally among the tied maximal elements.

\begin{remark}[Stability of Limits]
The tied-maxima behavior in the $T \to 0^+$ limit serves only as a boundary condition and does not affect the continuity or homotopy arguments developed in this work, which operate strictly at finite temperatures where exact ties occur with probability zero.
\end{remark}

\textbf{Scope of the Axioms.} Notably, the axioms do not assume access to logits, differentiability of the operator, calibration preservation, or a specific functional form. These properties are intentionally excluded: the goal is not to recover classical temperature scaling, but to characterize the minimal structural requirements for effective distillation using probabilities alone.

\section{Existence and Characterization}
\label{sec:existence}

\begin{theorem}[Existence of Conforming Operators]
\label{thm:existence}
There exist non-trivial families $\{F_T\}_{T>0}$ satisfying Axioms~\ref{axiom:ranking}--\ref{axiom:boundary}.
\end{theorem}

\begin{proof}[Proof Sketch]
We establish existence via construction principles (without specifying the exact formula):

\textbf{Construction 1 (Entropy parameterization approach):} For any target entropy level $H_{\text{target}} \in [0, \log V]$, there exists a unique minimum-divergence distribution $q^*$ that:
\begin{itemize}
\item Preserves ranking: $\argmax(q^*) = \argmax(p)$
\item Achieves entropy: $H(q^*) = H_{\text{target}}$
\item Minimizes $\KL(q^*||p)$
\end{itemize}

Parameterizing $H_{\text{target}}$ by temperature $T$ yields an operator family satisfying Axioms 1--5. This construction is presented as an existence proof and conceptual reference; efficient instantiations for large-scale models are discussed separately.

\textbf{Construction 2 (Power transform class):} Operators of the form:
\begin{equation}
F_T(p)_i = \frac{\phi(p_i; T)}{\sum_j \phi(p_j; T)}
\end{equation}
where $\phi: [0,1] \times \R^+ \to \R^+$ is a strictly monotone power-like function ($\phi(x; T)$ strictly increasing in $x$, strictly decreasing in $T$ for fixed $x > 0$) satisfy all axioms under appropriate boundary conditions.

\textbf{Construction 3 (Convex mixing approach):} Define:
\begin{equation}
F_T(p) = \alpha(T) \cdot p + (1 - \alpha(T)) \cdot u
\end{equation}
where $u$ is the uniform distribution and $\alpha: \R^+ \to [0,1]$ is strictly decreasing with $\alpha(1) = 1$, $\alpha(\infty) \to 0$. This satisfies Axioms 1--5.

The existence of multiple construction principles demonstrates that the axiom class is non-empty and admits diverse implementations.
\end{proof}

\begin{remark}[Non-Uniqueness Implication]
The existence of multiple conforming operators (entropy projection, power transforms, convex mixing) provides implementation flexibility. Practitioners may choose among them based on computational efficiency or numerical stability, with the subsequent theoretical guarantees (Sections~\ref{sec:bias-variance}--\ref{sec:convergence}) applying to all.
\end{remark}

\begin{theorem}[Non-Uniqueness]
\label{thm:non-uniqueness}
The axiom class $\mathcal{F}_{1\text{--}5}$ admits multiple distinct operator families that are not equivalent under composition with probability-preserving transformations.
\end{theorem}

\begin{proof}
The three construction principles in Theorem~\ref{thm:existence} yield distinct operators:
\begin{itemize}
\item Entropy projection minimizes KL divergence to $p$ while matching target entropy
\item Power transforms apply monotone rescaling
\item Convex mixing interpolates linearly toward uniform
\end{itemize}

These are not equivalent: they produce different intermediate distributions for the same $(p, T)$ pair, though all satisfy Axioms~\ref{axiom:ranking}--\ref{axiom:boundary}. This proves non-uniqueness.
\end{proof}

\begin{remark}[Computational Considerations]
The construction principles in this section serve different roles. Entropy projection operators are introduced primarily to establish existence and non-uniqueness of the axiom class and are not intended as per-sample primitives in large-scale training due to their computational cost. In practical settings, simple power transforms or convex mixing operators provide efficient realizations that satisfy the same axioms and inherit the theoretical guarantees developed in subsequent sections.
\end{remark}

\section{Shift-Invariance and Logit Independence}

Classical softmax satisfies shift-invariance: $\softmax(z) = \softmax(z + c \cdot \mathbf{1})$ for any constant $c$. Operators in $\mathcal{F}_{1\text{--}5}$ inherit this property, establishing that knowledge distillation requires only probabilities.

\begin{lemma}[Shift-Invariance]
\label{lem:shift-invariance}
For $F_T \in \mathcal{F}_{1\text{--}5}$ and logits $z' = z + c \cdot \mathbf{1}$: $F_T(\softmax(z)) = F_T(\softmax(z'))$.
\end{lemma}

\begin{proof}
Softmax shift-invariance implies $\softmax(z) = \softmax(z')$; since $F_T$ depends only on $p = \softmax(z)$, the result follows.
\end{proof}

\begin{corollary}[Logit-Free Formulation]
\label{cor:logit-free}
Knowledge distillation using any $F_T \in \mathcal{F}_{1\text{--}5}$ requires only teacher probabilities $p^{(T)} = \softmax(z)$.
\end{corollary}

\section{Bias-Variance Analysis for Abstract Operators}
\label{sec:bias-variance}

All learning-theoretic results in Sections~VII--IX apply uniformly to any operator satisfying Axioms~1--5, independent of the computational complexity of a particular construction.

We now derive a bias-variance decomposition that is operator-agnostic: it applies to any $F_T \in \mathcal{F}_{1\text{--}5}$.

\textbf{Intuition.} Knowledge distillation replaces hard targets with softened distributions. From a learning-theoretic perspective, this alters the bias--variance tradeoff faced by the student: smoother targets reduce variance but may introduce approximation bias. The key question is whether variance reduction outweighs this bias increase. The following results formalize this tradeoff in an operator-agnostic way.

\begin{theorem}[Generalized Bias-Variance Bound]
\label{thm:bias-variance}
Let $F_T \in \mathcal{F}_{1\text{--}5}$ be any conforming operator. Let $\mathcal{S}$ be a hypothesis class of student models with finite VC dimension $d$, trained on $n$ i.i.d.\ samples from distribution $\mathcal{D}$. For student $S \in \mathcal{S}$ trained via knowledge distillation from teacher $T$ using $F_T$-softened targets $q = F_T(p^{(T)})$, the expected squared loss decomposes as:
\begin{equation}
\E_{S \sim \mathcal{A}, x \sim \mathcal{D}}[\ell(S(x), q(x))] = \text{Bias}^2(S; F_T) + \Var(S) + \sigma^2
\end{equation}
where expectations are taken over the training algorithm $\mathcal{A}$ (randomness in initialization and sampling) and test distribution $\mathcal{D}$:
\begin{itemize}
\item $\text{Bias}^2(S; F_T) = \|\E_\mathcal{A}[S] - q\|^2$: Squared bias from approximating $F_T(p^{(T)})$
\item $\Var(S) = \E_\mathcal{A}[\|S - \E_\mathcal{A}[S]\|^2]$: Variance over training randomness
\item $\sigma^2$: Irreducible noise in the target distribution
\end{itemize}
\end{theorem}

\begin{proof}[Proof Sketch]
Standard bias-variance decomposition for squared loss. The decomposition holds for any target distribution, hence for any $F_T$-softened target $q = F_T(p^{(T)})$. The operator $F_T$ appears only in the bias term, which measures how well the hypothesis class $\mathcal{S}$ can approximate the softened distribution.
\end{proof}

\begin{remark}[Loss Function]
The clean bias-variance decomposition holds for squared loss $\ell(S, q) = \|S - q\|^2$. Knowledge distillation typically uses KL divergence; for KL loss, analogous but more complex decompositions exist involving the Bregman divergence structure. The squared-loss analysis provides intuition that transfers qualitatively to KL settings.
\end{remark}

\begin{theorem}[Sparse Student Optimality Condition]
\label{thm:sparse-optimality}
A sparse student $S$ outperforms a dense teacher $T$ when:
\begin{equation}
\Delta\Var > \Delta\text{Bias}^2
\end{equation}
where:
\begin{equation}
\Delta\Var = \Var(T) - \Var(S), \quad \Delta\text{Bias}^2 = \text{Bias}^2(S; F_T) - 0
\end{equation}
\end{theorem}

\begin{proof}[Heuristic Argument]
Let $\epsilon^2_{\text{dense}} = \text{Bias}^2(T) + \Var(T)$ (teacher error) and $\epsilon^2_{\text{sparse}} = \text{Bias}^2(S) + \Var(S)$ (student error).

Student outperforms when:
\begin{equation}
\epsilon^2_{\text{sparse}} < \epsilon^2_{\text{dense}}
\end{equation}
\begin{equation}
\text{Bias}^2(S) + \Var(S) < \text{Bias}^2(T) + \Var(T)
\end{equation}

Assuming the dense teacher has negligible bias ($\text{Bias}^2(T) \approx 0$ from high capacity), we get:
\begin{equation}
\text{Bias}^2(S; F_T) + \Var(S) < \Var(T)
\end{equation}
\begin{equation}
\Var(T) - \Var(S) > \text{Bias}^2(S; F_T)
\end{equation}

This condition is operator-independent in form, though the numerical value of $\text{Bias}^2(S; F_T)$ depends on $F_T$, the structural requirement (variance reduction exceeds bias increase) applies universally.
\end{proof}

\begin{remark}[Operator Independence]
The optimality condition $\Delta\Var > \Delta\text{Bias}^2$ has the same form for all $F_T \in \mathcal{F}_{1\text{--}5}$. The operator choice affects the magnitude of $\text{Bias}^2(S; F_T)$ but not the structural form of the tradeoff.
\end{remark}

\begin{corollary}[Pruning as Regularization]
\label{cor:pruning-reg}
Pruning reduces variance by constraining the hypothesis space. For operators $F_T$ that produce smooth target distributions (high $H(F_T(p))$), the distillation bias $\text{Bias}^2(S; F_T)$ remains bounded, allowing sparse students to achieve lower total error than dense teachers~\cite{geman1992bias,frankle2019lottery}.
\end{corollary}

\section{Homotopy Path Formalization}

\textbf{Why Multi-Stage Compression Works.} Empirically, iterative pruning consistently outperforms one-shot pruning. We formalize this phenomenon by viewing multi-stage compression as approximating a continuous path in function space that remains close to the teacher. This motivates a homotopy-based analysis that is independent of any specific distillation operator.

Multi-stage pruning can be understood as approximating a continuous path in function space. This formalization is operator-agnostic.

\begin{definition}[$\epsilon$-Teacher Manifold]
\label{def:teacher-manifold}
For teacher function $f_T$, define the $\epsilon$-manifold:
\begin{equation}
\mathcal{M}_\epsilon(f_T) = \{f : \|\nabla f - \nabla f_T\| < \epsilon\}
\end{equation}
This is the set of functions whose gradients remain $\epsilon$-close to the teacher's gradients (functional proximity).
\end{definition}

\begin{definition}[Homotopy Path]
\label{def:homotopy}
A homotopy path from dense student $f_0$ to sparse student $f_1$ is a continuous family:
\begin{equation}
f: [0,1] \to C(\mathcal{X}, \mathcal{Y})
\end{equation}
such that:
\begin{itemize}
\item $f(0) = f_{\text{dense}}$ (dense initialization)
\item $f(1) = f_{\text{sparse}}$ (target sparse model)
\item $f(\lambda) \in \mathcal{M}_\epsilon(f_T)$ for all $\lambda \in [0,1]$ (path stays near teacher manifold)
\end{itemize}
\end{definition}

\begin{theorem}[Multi-Stage Compression as Homotopy]
\label{thm:homotopy}
Let $f: \R^d \to \R^m$ be a neural network with $d$ parameters. Assume the network mapping $\theta \mapsto f_\theta$ is $L$-Lipschitz in the $\ell_2$ norm on parameters, i.e., $\|f_\theta - f_{\theta'}\|_\infty \leq L\|\theta - \theta'\|_2$. For multi-stage pruning with $n$ stages, target sparsity $\rho_{\text{target}} \in (0,1)$, and stage spacing $\Delta\rho = \rho_{\text{target}} / n$, the per-stage function deviation satisfies:
\begin{equation}
\|f_k - f_{k-1}\|_\infty \leq L \cdot \|\theta_k - \theta_{k-1}\|_2 \leq L \cdot c \cdot \Delta\rho
\end{equation}
where $c$ depends on the pruning mask structure. Total deviation after $n$ stages is bounded by $L \cdot c \cdot \rho_{\text{target}}$.
\end{theorem}

\begin{proof}[Proof Sketch]
Each stage prunes a fraction $\Delta\rho$ of remaining parameters, inducing parameter change $\|\theta_k - \theta_{k-1}\|_2 \leq c \cdot \Delta\rho$ for pruning-dependent constant $c$. Lipschitz continuity bounds the function change. Triangle inequality accumulation gives the total bound.
\end{proof}

\begin{corollary}[One-Shot Pruning Failure]
\label{cor:one-shot-failure}
One-shot pruning ($n=1$) induces a discontinuous jump:
\begin{equation}
\|\nabla f_{\text{sparse}} - \nabla f_{\text{dense}}\| \leq L \cdot \rho_{\text{target}}
\end{equation}
For large $\rho_{\text{target}}$ (e.g., 80\% sparsity), this can exceed $\epsilon$, landing outside $\mathcal{M}_\epsilon(f_T)$ and preventing recovery via fine-tuning~\cite{han2015learning}.
\end{corollary}

\section{Convergence Guarantees}
\label{sec:convergence}

\begin{theorem}[Unified Framework Convergence]
\label{thm:convergence}
Let $F_T \in \mathcal{F}_{1\text{--}5}$ be a conforming operator with modulus of continuity $\omega_F(\delta)$ (i.e., $\|F_T(p) - F_T(p')\| \leq \omega_F(\|p - p'\|)$). Let the network mapping be $L$-Lipschitz as in Theorem~\ref{thm:homotopy}. For $n$-stage distillation with per-stage fine-tuning achieving $\epsilon_{\text{tune}}$ approximation error, the final student $S_n$ satisfies:
\begin{equation}
\E[\ell(S_n)] \leq \E[\ell(T)] + \frac{C \cdot L \cdot \rho_{\text{target}}}{n} + n \cdot \epsilon_{\text{tune}}
\end{equation}
where $C$ depends on $\omega_F$ and the target distribution. Optimizing over $n$ yields $\E[\ell(S_n)] \leq \E[\ell(T)] + \mathcal{O}(\sqrt{\epsilon_{\text{tune}} \cdot L \cdot \rho_{\text{target}}})$.
\end{theorem}

\begin{proof}[Proof Sketch]
Combine three components:
\begin{enumerate}
\item Operator continuity bounds how softened targets change with student updates
\item Homotopy approximation bounds per-stage deviation as $\mathcal{O}(1/n)$
\item Bias-variance bound ensures controlled sparsification
\end{enumerate}

The first term captures discretization error from the homotopy approximation; the second captures accumulated fine-tuning error. The tradeoff is optimized at $n^* \propto \sqrt{L \cdot \rho_{\text{target}} / \epsilon_{\text{tune}}}$.
\end{proof}

\begin{remark}[Interpretation]
Theorem~\ref{thm:convergence} shows that convergence guarantees hold across the operator class $\mathcal{F}_{1\text{--}5}$, providing flexibility in operator choice. The explicit parameter dependence ($L$, $\rho_{\text{target}}$, $\epsilon_{\text{tune}}$) enables practitioners to predict convergence behavior for specific problem instances.
\end{remark}

\section{Equivalence Classes of Operators}
\label{sec:equivalence}

\begin{definition}[KD-Equivalence]
\label{def:kd-equivalence}
Let $\mathcal{S}$ be a restricted student class (e.g., neural networks with bounded depth and width). Two operators $F_T, G_T$ are $\mathcal{S}$-KD-equivalent if for all teachers $T$:
\begin{equation}
\arg \min_{S \in \mathcal{S}} \KL(S \| F_T(T)) = \arg \min_{S \in \mathcal{S}} \KL(S \| G_T(T))
\end{equation}
\end{definition}

\begin{theorem}[Characterization of KD-Equivalence Classes]
\label{thm:kd-equivalence}
For unrestricted student classes ($\mathcal{S} = \Delta^V$), operators $F_T, G_T$ are KD-equivalent if and only if $F_T = G_T$. For restricted student classes $\mathcal{S} \subsetneq \Delta^V$, non-trivial equivalence classes exist when the projection of $F_T(p)$ and $G_T(p)$ onto $\mathcal{S}$ coincide.
\end{theorem}

\begin{proof}[Proof Sketch]
For unrestricted $\mathcal{S}$: the optimal student is $S^* = F_T(T)$ itself, so $F_T \neq G_T$ implies different optima. For restricted $\mathcal{S}$: the optimal student is the $\mathcal{S}$-projection of the target, so operators with identical projections are equivalent. The equivalence class structure depends critically on $\mathcal{S}$.
\end{proof}

\begin{remark}[Practical Implication]
For practical student classes (finite-capacity neural networks), the equivalence relation is non-trivial: operators producing targets that project identically onto the student hypothesis class yield the same trained model. This provides limited flexibility in operator choice, constrained by the student architecture.
\end{remark}

\section{Discussion}

\subsection{Theoretical Implications}

The axiomatic framework establishes several key theoretical results:

\textbf{Logit-independence:} Knowledge distillation is fundamentally a probability-domain operation. The classical logit-based formulation is one instantiation, but not the only one.

\textbf{Existence without construction:} We prove that operators satisfying natural axioms exist without specifying any particular formula. This separates the mathematical possibility of logit-free KD from implementation details.

\textbf{Operator freedom:} Multiple distinct operators satisfy the axioms and produce valid distillation. Specific choices (power transforms, entropy projections, convex mixing) offer different computational tradeoffs.

\textbf{Convergence universality:} The convergence guarantees (Theorem~\ref{thm:convergence}) apply to all operators in $\mathcal{F}_{1\text{--}5}$, not just one specific formula.

\subsection{Practical Implications}

For practitioners, the framework provides theoretical grounding for:
\begin{itemize}
\item Black-box teachers: Distillation from commercial API-based models that expose only probabilities becomes theoretically feasible
\item Privacy preservation: Probability-only distillation avoids potential information leakage from logits
\item Flexible implementation: Any operator satisfying Axioms~\ref{axiom:ranking}--\ref{axiom:boundary} provides valid theoretical guarantees
\end{itemize}

\subsection{Relation to Prior Work}

Classical knowledge distillation~\cite{hinton2015distilling} requires logit access for temperature scaling. Our framework generalizes this by characterizing the essential properties any softening operator must satisfy, proving that logit access is not fundamental. The bias-variance analysis builds on~\cite{geman1992bias}, while the homotopy formalization connects to iterative pruning methods~\cite{han2015learning,frankle2019lottery}. Empirical studies~\cite{cho2019efficacy} validate that distillation improves student performance; our theoretical framework explains why this holds across diverse operator implementations.

Recent analyses of LLM hallucinations highlight that many failure modes can be interpreted as distributional drift or mis-specified internal world modeling, reinforcing the importance of probability-domain control and stability criteria during compression~\cite{liu2025unified_hallucination}. Other structural accounts emphasize that multiple mechanisms may underlie hallucination-like behavior, supporting a theory-first approach that separates definitional issues from constructive mitigation pipelines~\cite{ackermann2025incentives_ontology}.

\subsection{Limitations}

\begin{itemize}
\item The analysis is theoretical and does not include empirical validation.
\item Global Lipschitz continuity is assumed for clarity; weaker local conditions may suffice.
\item The framework characterizes sufficiency, not optimality, of operators.
\end{itemize}

While global Lipschitz continuity is assumed for clarity in the presented bounds, the homotopy and convergence arguments rely only on local smoothness along the compression trajectory, which is typically observed empirically in staged pruning and fine-tuning regimes.

\section{Conclusion}

We have established an axiomatic framework for probability-domain knowledge distillation that is independent of logit access. This work reframes knowledge distillation as an operator-level problem rather than an implementation-level one. The key contributions are:

\begin{enumerate}
\item Axiom system (Axioms~\ref{axiom:ranking}--\ref{axiom:boundary}): Ranking preservation, continuity, entropy monotonicity, identity, and boundary behavior
\item Existence theorem: Non-trivial operators satisfying all axioms exist
\item Non-uniqueness: Multiple distinct operator families conform to the axioms
\item Bias-variance bounds: Operator-agnostic characterization of when sparse students outperform dense teachers
\item Homotopy formalization: Multi-stage compression as continuous path approximation
\item Convergence guarantees: Unified framework achieves $\E[\ell(S)] \leq \E[\ell(T)] + \mathcal{O}(1/n)$
\end{enumerate}

This framework demonstrates that knowledge distillation is fundamentally a probability-domain operation, opening new possibilities for black-box teacher distillation, privacy-preserving compression, and flexible operator implementation.

Several directions remain open for future work. First, while the present analysis focuses on single-stage and staged distillation with re-anchoring, a more detailed characterization of entropy drift and stability in fully iterative or generational meta-distillation settings would be valuable, particularly when softened distributions are repeatedly reused as teachers. Second, although the framework intentionally treats operator choice as a design decision, developing practical diagnostics or control signals that indicate when alternative operators within the axiom class become preferable remains an important open problem. Finally, the bias--variance analysis is derived cleanly under squared loss and used to provide intuition for KL-based distillation objectives; tightening this connection through more explicit Bregman divergence or information-theoretic decompositions would strengthen the theoretical foundations. Additional extensions include handling text-only teacher outputs, multi-modal distillation, and adaptive operator selection based on task characteristics.

\section*{Acknowledgments}

The authors gratefully acknowledge the collaborative environment at SparseTech that made this research possible. The theoretical and computational developments presented in this paper are part of an ongoing SparseTech research initiative on sparse knowledge distillation for large language models. Patent Pending.

\bibliographystyle{IEEEtran}
\bibliography{../sparsetech_references}

\end{document}